\documentclass{article}

\input{preamble}

\begin{document}
\twocolumn[

\aistatstitle{Variational Sequential Monte Carlo}

\aistatsauthor{Christian A. Naesseth \And Scott W. Linderman \And  Rajesh Ranganath \And David M. Blei }

\aistatsaddress{Link\"oping University \And Columbia University \And New York University \And Columbia University } 
]

\begin{abstract}
  Many recent advances in large scale probabilistic inference rely
  on variational methods. The success of variational approaches
  depends on \begin{enumerate*}[label=(\roman*)]
\item formulating a flexible parametric family of distributions, and
\item optimizing the parameters to find the member of this family that
  most closely approximates the exact posterior.
\end{enumerate*} In this paper we present a new approximating family
of distributions, the \gls{VSMC} family, and show how to optimize it in
variational inference. \gls{VSMC} melds \gls{VI} and \gls{SMC},
providing practitioners with flexible, accurate, and powerful Bayesian
inference.  The \gls{VSMC} family is a variational family that can approximate
the posterior arbitrarily well, while still allowing for efficient
optimization of its parameters. We demonstrate its utility on state
space models, stochastic volatility models for financial data, and 
deep Markov models of brain neural circuits.
\end{abstract}

\author{
  Christian A. Naesseth\\
  Link\"oping University\\
  \texttt{christian.a.naesseth@liu.se} \\
  \And
  Scott W. Linderman\\
  Columbia University\\
  \texttt{scott.linderman@columbia.edu} \\
  \AND
  Rajesh Ranganath\\
  New York University\\
  \texttt{rajeshr@cims.nyu.edu} \\
  \And
  David M. Blei\\
  Columbia University\\
  \texttt{david.blei@columbia.edu} \\
}

\glsresetall

\section{Introduction}\label{sec:intro}
Complex data like natural images, text, and medical records
require sophisticated models and algorithms. Recent advances in these
challenging domains have relied upon \gls{VI}~\citep{Kingma2014, Hoffman2013, ranganath2016deep}. 
Variational inference excels in quickly approximating the model posterior,
yet these approximations are only useful insofar as they are accurate.
The challenge is to balance faithful posterior approximation and fast optimization.

We present a new approximating family of distributions called
\gls{VSMC}. \gls{VSMC} blends \gls{VI} and \gls{SMC}
\citep{stewart1992,GordonSS:1993,kitagawa1996monte}, providing
practitioners with a flexible, accurate, and powerful approximate
Bayesian inference algorithm.  \gls{VSMC} is an efficient algorithm that
can approximate the posterior arbitrarily well.

Standard \gls{SMC} approximates a posterior distribution of latent
variables with $N$ weighted particles iteratively drawn from a
proposal distribution. The idea behind \emph{variational} \gls{SMC} is
to view the parameters of the proposal as indexing a family of
distributions over latent variables.  Each distribution in this
variational family corresponds to a particular choice of proposal; to
sample the distribution, we run \gls{SMC} to generate a set
of particles and then randomly select one with probability
proportional to its weight.  Unlike typical variational families, the
\gls{VSMC} family trades off fidelity to the posterior with
computational complexity: its accuracy increases with the number of
particles $N$, but so does its computational cost.

\begin{figure*}[t]
  \begin{center}
    \includegraphics{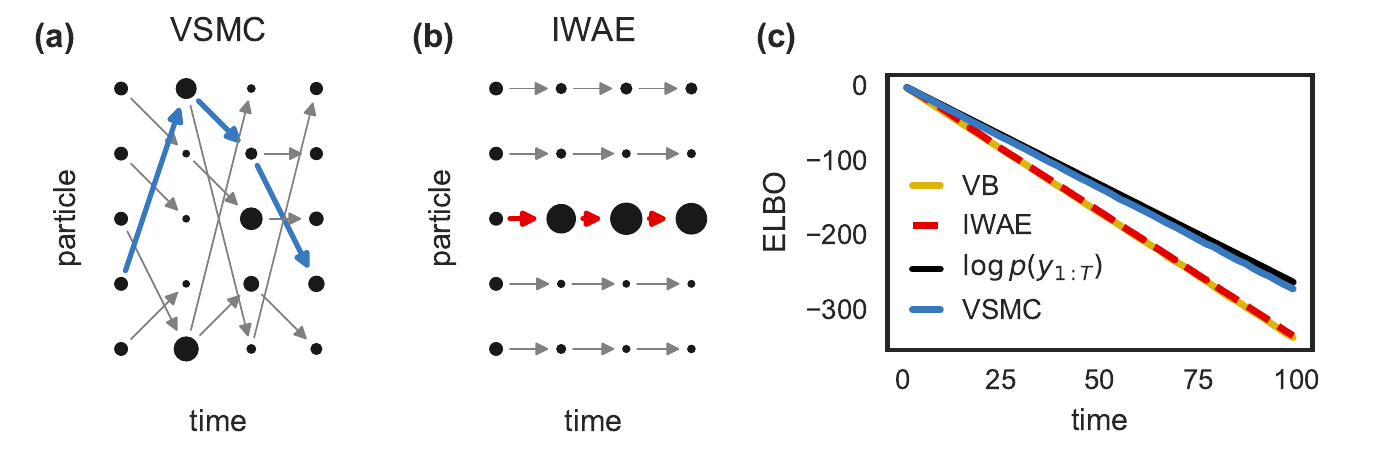}
  \end{center}
  \vspace{-5mm}
  \caption{\textit{Comparing \acrshort{VSMC} and the \acrshort{IWAE}.}
    (a) \acrshort{VSMC} constructs a weighted set of particle trajectories using SMC
    and then samples one according to the final weight. Here, the size of the dot
    is proportional to the weight,~$w_{t}^{i}$; the gray arrows denote the
    ancestors,~$a_{t-1}^{i}$; and the blue arrows denote the chosen path,~$b_{1:T}$.
    (b) \acrshort{IWAE} does the same, but without resampling. This leads to 
    particle degeneracy as time increases---only one particle has nonneglible weight
    at time~$T$. 
    (c) The ELBO suffers from this degeneracy: all are comparable when~$T$ is small,
    but as time increases the \acrshort{IWAE}
    provides minimal improvement over standard \acrshort{VB}, whereas \acrshort{VSMC}
    still achieves nearly the true marginal likelihood.}
  \label{fig:fig1}
\end{figure*}

We develop the \gls{VSMC} approximating family, 
derive its corresponding variational lower bound,
and design a stochastic gradient ascent algorithm to optimize its
parameters. We connect \gls{VSMC} to the \gls{IWAE} \citep{Burda2016} and show that the
\gls{IWAE} lower bound is a special case of the \gls{VSMC} bound.
As an illustration, consider approximating the
following posterior with latent variables $x_{1:T}$ and observations
$y_{1:T}$,
\begin{align*}
{p(x_{1:T}\mid y_{1:T})=\prod_{t=1}^T\mathcal{N}(x_t\g
  0,1) \, \mathcal{N}(y_t\g x_t^2,1)/p(y_{1:T})}.
\end{align*}
This is a toy Gaussian \gls{SSM} where the observed value at 
each time step depends on the square of the latent state.
Figure~\ref{fig:fig1}c shows the
approximating power of \gls{VSMC} versus that of the \gls{IWAE} and of
standard \gls{VB}. As the length of the sequence~$T$ increases,
na\"ive importance sampling effectively collapses to use only a single 
particle. \gls{VSMC} on the other hand maintains a diverse set of 
particles and thereby achieves a significantly tighter lower bound of 
the log-marginal likelihood $\log p(y_{1:T})$.

We focus on inference in state space and time series models, but emphasize that 
\gls{VSMC} applies to any sequence of probabilistic models, just like 
standard \gls{SMC} \citep{del2006sequential,doucet2009tutorial,Naesseth2014}.

In Section~\ref{sec:expts}, we demonstrate the advantages of \gls{VSMC} on both simulated and real 
data. First, we show on simulated linear Gaussian \gls{SSM} data that 
\gls{VSMC} can outperform the (locally) 
optimal proposal \citep{doucet2001introduction,doucet2009tutorial}.
Then we compare \gls{VSMC} with \gls{IWAE} for a stochastic volatility model 
on exchange rates from financial markets. 
We find that \gls{VSMC} achieves better posterior inferences
and learns more efficient proposals.
Finally, we study recordings of macaque monkey neurons using a 
probabilistic model based on recurrent neural networks. \gls{VSMC} reaches the same 
accuracy as \gls{IWAE}, but does so with less computation.


\paragraph{Related Work} 
Much effort has been dedicated to learning good 
proposals for \gls{SMC} \citep{Cornebise2009}. \citet{Guarniero2015} adapt proposals through iterative refinement.
\citet{naessethLS2015nested} uses a Monte Carlo approximation to the (locally) optimal proposal \citep{doucet2009tutorial}.
\citet{Gu2015} learn proposals by minimizing the \gls{KL} from the 
posterior to
proposal using SMC samples; this strategy can suffer from high variance when the initial SMC proposal is poor. \citet{Paige2016} learn proposals by 
forward simulating and inverting the model. In contrast to all these methods, \gls{VSMC} optimizes the proposal directly with respect to KL divergence from the \gls{SMC}
sampling process to the posterior.

\gls{VSMC} uses auxillary variables 
in a posterior approximation. This relates to work in \gls{VI}, such as
Hamiltonian VI \citep{Salimans2015}, variational Gaussian processes \citep{Tran2016}, 
hierarchical variational models \citep{Ranganath2016}, and deep auxiliary 
variational auto-encoders \citep{Maaloe2016}.
Another approach uses a sequence of invertible functions 
to transform a simple variational approximation to a complex one 
\citep{Rezende2015,dinh2014nice}.
All of these rich approximations can be 
embedded inside \gls{VSMC} to build more flexible proposals.

\citet{Archer2015, Johnson2016} develop variational inference
for state space models with conjugate dynamics, while \citet{Krishnan2017}
develop variational approximations for models with nonlinear dynamics and additive Gaussian noise.
In contrast, \gls{VSMC} is agnostic to the distributional choices in
the dynamics and noise. 

Importance weighted auto-encoders \citep{Burda2016} obtain the same 
lower bound as \gls{VIS}, a special case of \gls{VSMC}. However, 
\gls{VIS} provides a new interpretation that enables a more accurate 
variational approximation; this relates to another interpretation of 
\gls{IWAE} by \citet{Cremer2017,Bachman2015}. 
Variational particle approximations \citep{saeedi2014variational} also 
provide variational approximation that improve with the number of particles, 
but they are restricted to discrete latent variables. 

Finally, the log-marginal likelihood lower bound 
\eqref{eq:lb} was developed concurrently and independently by 
\citet{Maddison2017} and \citet{Le2017aesmc}. 
The difference with our work
lies in how we derive the bound and the implications we explore. 
\citet{Maddison2017,Le2017aesmc} derive the bound
using Jensen's inequality on the \gls{SMC} expected log-marginal likelihood
estimate, focusing on approximate marginal likelihood estimation of
model parameters. 
Rather, we derive \eqref{eq:lb} as a tractable lower bound to
the exact \gls{ELBO} for the new variational family \gls{VSMC}. 
In addition to a lower bound on the log-marginal likelihood,
this view provides a new variational approximation to the posterior.


\section{Background}
We begin by introducing the foundation for \glsreset{VSMC}\gls{VSMC}.  
Let $p(x_{1:t},y_{1:t})$ be a
sequence of probabilistic models for latent (unobserved)~$x_{1:t}$ and
data~$y_{1:t}$, with $t=1,\ldots,T$. In Bayesian inference, we are interested in
computing the posterior distribution~$p(x_{1:T} \given y_{1:T})$.  Two
concrete examples, both from the time-series literature, are hidden
Markov models and state space models \citep{cappe2005inference}.  In
both cases, the joint density factorizes as
\begin{align*}
  p(x_{1:T},y_{1:T}) &= f(x_1) \prod_{t=2}^T f(x_t \given x_{t-1})
                       \prod_{t=1}^T g(y_t \given x_t),
\end{align*}
where $f$ is the prior on $x$, and $g$ is the observation (data)
distribution. For most models computing the
posterior $p(x_{1:T} \given y_{1:T})$ is computationally intractable, and we need approximations
such as \gls{VI} and \gls{SMC}. Here we construct posterior
approximations that combine these two ideas.

In the following sections, we review \acrlong{VI} and \acrlong{SMC},
develop a variational approximation based on the samples generated by
\gls{SMC}, and develop a tractable objective to improve the quality of
the \gls{SMC} variational approximation. For concreteness, we focus on
the state space model above. But we emphasize that \gls{VSMC} applies 
to any sequence of probabilistic models, just like standard \gls{SMC} 
\citep{del2006sequential,doucet2009tutorial,Naesseth2014}.



\paragraph{Variational Inference}\label{sec:vi}
In \acrlong{VI} we postulate an 
approximating family of distributions with variational parameters~$\lambda$, $q(x_{1:T};\lambda)$. Then we 
minimize a divergence, often the \gls{KL} 
divergence, between the approximating family and the posterior so that~${q(x_{1:T};\lambda) \approx p(x_{1:T} \given y_{1:T})}$. This minimization is 
equivalent to maximizing the \gls{ELBO} \citep{Jordan1999},
\begin{align}
\Ls(\lambda) &= \E_{q(x_{1:T}; \lambda)}\left[\log p(x_{1:T},y_{1:T})-\log 
q(x_{1:T}; \lambda)\right].\label{eq:elbo}
\end{align}
\gls{VI} turns posterior inference into an optimization problem.


\paragraph{Sequential Monte Carlo}\label{sec:smc}
\gls{SMC} is a sampling method designed to approximate a
sequence of distributions, $p(x_{1:t} \given y_{1:t})$ for
$t = 1 \ldots T$ with special emphasis on the posterior
$p(x_{1:T} \given y_{1:T})$.  For a thorough
introduction to \gls{SMC} see
\citet{doucet2009tutorial,doucet2001introduction,schonldwnsd2015}.

To approximate $p(x_{1:t} \given y_{1:t})$ \gls{SMC} uses weighted samples,
\begin{align}
  p(x_{1:t} \given y_{1:t}) \approx \widehat p(x_{1:t}\given y_{1:t})
  \triangleq
 \sum_{i=1}^N \frac{w_t^i}{\sum_\ell 
w_t^\ell} \delta_{x_{1:t}^i}, \label{eq:smcapprox}
\end{align}
where $\delta_X$ is the Dirac measure at $X$. 

We construct the weighted set of particles sequentially for
$t=1,\ldots,T$. At time $t=1$ we use standard importance sampling
$x_1^i\sim r(x_1)$.
For $t>1$, we start each step by \emph{resampling} 
auxiliary \emph{ancestor variables} ${a_{t-1}^i \in \{1, \ldots, N\}}$ with 
probability proportional to the importance weights $w_{t-1}^{j}$; next
we propose new values, append them to the end of the trajectory,
and reweight as follows:
\begin{align*}
 & &  &\textit{resample } & & a_{t-1}^i \sim \textrm{Categorical}(\nicefrac{w_{t-1}^j}{\sum_\ell w_{t-1}^\ell})
 \\
 & &  &\textit{propose } & & x_{t}^i \sim r(x_t\given x_{t-1}^{a_{t-1}^i}),
  \\
 & &  &\textit{append } & &  x_{1:t}^i = (x_{1:t-1}^{a_{t-1}^i},x_t^i),
  \\
 & &   &\textit{reweight } & & w_t^i = 
 \nicefrac{f(x_t^i\given x_{t-1}^{a_{t-1}^i}) \, g(y_t\given x_t^i)}{r(x_t^i\given x_{t-1}^{a_{t-1}^i})}.
\end{align*}
We refer to the final particles (samples) $x_{1:T}^i$ as
\emph{trajectories}. Panels (a) and (b) of Figure~\ref{fig:fig1} show
sets of weighted trajectories. The size of the dots represents
the weights $w_t^i$ and the arrows represent the ancestors
$a_{t-1}^i$.  Importance sampling omits the resampling step, so each
ancestor is given by the corresponding particle for the preceding time
step.

The trajectories $x_{1:T}^i$ and weights $w_T^i$ define the \gls{SMC}
approximation to the posterior. Critically, as we increase the number
of particles, the posterior approximation becomes arbitrarily
accurate.  \gls{SMC} also yields an unbiased estimate of the marginal
likelihood,
\begin{align}
  \label{eq:smc_margll}
  \widehat{p}(y_{1:T})
  &= \prod_{t=1}^T \frac{1}{N}\sum_{i=1}^N w_t^i.
\end{align}
This estimate will play an important role in the \gls{VSMC} objective.

The proposal distribution ${r(x_t \given x_{t-1})}$ is the key design
choice. A common choice is the model prior $f$---it is known as
the~\gls{BPF} \citep{GordonSS:1993}. However, proposing from the prior often leads to a
poor approximation for a small number of particles, especially if $x_t$ is
high-dimensional. Variational \gls{SMC}
addresses this shortcoming; it learns parameterized
proposal distributions for efficient inference.



\section{Variational Sequential Monte Carlo}\label{sec:vsmc}

We develop \gls{VSMC}, a new class of variational approximations based
on \gls{SMC}. 
We first define how to sample from the \gls{VSMC} family and then derive its
distribution. Though generating samples is straightforward,
the density is intractable. To this end, we derive a tractable
objective, a new lower bound to the \gls{ELBO}, that is amenable to
stochastic optimization. Then, we present an
algorithm to fit the variational parameters. Finally, we explore how to 
learn model parameters using \acrlong{VEM}.

To sample from the \gls{VSMC} family, we
run \gls{SMC} (with the proposals parameterized by variational parameters
$\lambda$) and then sample once
from the empirical approximation of the
posterior~\eqref{eq:smcapprox}. Because the proposals
$r(x_t\mid x_{t-1}\g\lambda)$ depend on $\lambda$, so does the
\gls{SMC} empirical approximation.  Algorithm~\ref{alg:vsmc} summarizes 
the generative process for the \gls{VSMC} family.

\begin{algorithm}[t]
\caption{Variational Sequential Monte Carlo}\label{alg:vsmc}
\begin{algorithmic}[1]
\REQUIRE Targets $p(x_{1:t},y_{1:t})$, proposals 
$r(x_t \given x_{t-1}\g \lambda)$, and number of particles $N$. 
\vspace{.5em}
\FOR{$i=1 \ldots N$}
\STATE Simulate $x_1^i$ from $r(x_1\g\lambda)$
\STATE Set $w_1^i = \nicefrac{f(x_1^i) \, g(y_1 \given x_1^i)}{r(x_1^i\g\lambda)}$
\ENDFOR
\FOR{$t=2 \ldots T$}
\FOR{$i=1 \ldots N$}
\STATE Simulate $a_{t-1}^i$ with $\prb(a_{t-1}^i=j) = 
\frac{w_{t-1}^j}{\sum_\ell w_{t-1}^\ell}$
\STATE Simulate $x_t^i$ from $r(x_t \given x_{t-1}^{a_{t-1}^i}\g\lambda)$
\STATE Set $x_{1:t}^i = (x_{1:t-1}^{a_{t-1}^i},x_t^i)$
\STATE Set $w_t^i = 
\nicefrac{f(x_t^i\given x_{t-1}^{a_{t-1}^i}) \, g(y_t\given x_t^i)}{ 
  r(x_t^i\given x_{t-1}^{a_{t-1}^i}\g\lambda)}$
\ENDFOR
\ENDFOR
\STATE Simulate $b_T$ with $\prb(b_{T}=j) = 
\nicefrac{w_{T}^j}{\sum_\ell w_{T}^\ell}$
\RETURN $x_{1:T} \triangleq x_{1:T}^{b_T}$
\end{algorithmic}
\end{algorithm}
The variational distribution $q(x_{1:T}\g\lambda)$
marginalizes out all the variables produced in the sampling process,
save for the output sample $x_{1:T}$.  This marginal comes
from the joint distribution of all variables generated by \gls{VSMC},
\begin{align}
  &\widetilde\phi(x_{1:T}^{1:N},a_{1:T-1}^{1:N},b_T \g \lambda) =
    \underbrace{\bigg[\prod_{i=1}^N r(x_1^i\g\lambda)\bigg]}_{\textit{step 2}} \cdot \nonumber\\
    &\cdot
    \prod_{t=2}^T \prod_{i=1}^N
    \underbrace{\bigg[\frac{w_{t-1}^{a_{t-1}^i}}{\sum_\ell w_{t-1}^\ell} }_{\textit{step 7}}
    \underbrace{r(x_t^i \given x_{t-1}^{a_{t-1}^i}\g\lambda) \bigg]}_{\textit{step 8}}
    \underbrace{\bigg[\frac{w_T^{b_T}}{\sum_{\ell}w_T^\ell} \bigg]}_{\textit{step 13}}.
\label{eq:vsmc-all}
\end{align}
(We have annotated this equation with the steps from the algorithm.)
In this joint, the final output sample is defined by extracting the
$b_T$-th trajectory ${x_{1:T} = x_{1:T}^{b_T}}$.  Note that the data
$y_{1:T}$ enter via the weights and (optionally) the proposal
distribution.  This joint density is easy to calculate, but for
variational inference we need the marginal distribution
of~$x_{1:T}$. We derive this next.

Let~$b_t \triangleq a_t^{b_{t+1}}$ for
$t \leq T-1$ denote the ancestors for the trajectory $x_{1:T}$ returned by 
Algorithm~\ref{alg:vsmc}. Furthermore, let~$\neg b_{1:T}$ be all 
particle indices \emph{not} equal to $(b_1,\ldots,b_T)$, \ie exactly all the 
particles that were not returned by Algorithm~\ref{alg:vsmc}. 
Then the marginal distribution of $x_{1:T} = x_{1:T}^{b_{1:T}} = (x_1^{b_1},x_2^{b_2},\ldots, 
x_{T}^{b_T})$ is given by the following proposition.
\begin{prop}
The \gls{VSMC} 
approximation on $x_{1:T}$ is
\begin{align}
&q(x_{1:T} \given y_{1:T} \g\lambda) \nonumber \\
&= p(x_{1:T},y_{1:T})  \, 
\E_{\widetilde\phi\left(x_{1:T}^{\neg b_{1:T}},a_{1:T-1}^{\neg b_{1:T-1}} \g \lambda \right)}
\left[\widehat p(y_{1:T})^{-1}\right].\label{eq:var-dist}
\end{align}
\label{prop:dist}
\end{prop}
\begin{proof}
\vspace{-8mm}
See the supplementary material \ref{sec:pf:dist}.
\end{proof}
This has an intuitive form: the density of the variational posterior
is equal to the exact joint times the expected inverse of the normalization
constant (c.f.~\eqref{eq:smc_margll}). While we can estimate this
expectation with Monte Carlo, it yields a biased estimate
of~${\log q(x_{1:T} \given y_{1:T}; \lambda)}$ and the 
\gls{ELBO}~\eqref{eq:elbo}. 

\paragraph{The surrogate \gls{ELBO}.}  To derive a
tractable objective, we develop a lower bound to the \gls{ELBO} that
is also amenable to stochastic optimization.  It is
\begin{align}
\widetilde{\mathcal{L}}(\lambda) &\triangleq \sum_{t=1}^T\E_{\widetilde\phi(x_{1:t}^{1:N},a_{1:t-1}^{1:N}\g\lambda)}\left[ 
\log \left( \frac{1}{N}\sum_{i=1}^N w_t^i\right) \right] \nonumber \\
& = \E
\left[\log \widehat{p}(y_{1:T})\right]
\label{eq:lb}
\end{align}
We call $\widetilde{\mathcal{L}}(\lambda)$ the \emph{surrogate
  \gls{ELBO}}.  It is a lower bound to the true \gls{ELBO} for
\gls{VSMC} or, equivalently, an upper bound on the \gls{KL}
divergence.  The following theorem formalizes this fact:
\begin{theorem}[Surrogate \gls{ELBO}]
The surrogate \gls{ELBO} \eqref{eq:lb},
is a lower bound to the \gls{ELBO} \eqref{eq:elbo} 
when $q$ is defined by \eqref{eq:var-dist}, \ie
\begin{align*}
\log p(y_{1:T}) \geq \mathcal{L}(\lambda) \geq \widetilde{\mathcal{L}}(\lambda).
\end{align*}
\label{thm:lb}
\end{theorem}
\begin{proof}
\vspace{-10mm}
See the supplementary material \ref{sec:pf:lb}.
\end{proof}
The surrogate \gls{ELBO} is the expected \gls{SMC}
log-marginal likelihood estimate. We can estimate it unbiasedly as a 
byproduct of sampling from the \gls{VSMC} variational approximation
(Algorithm~\ref{alg:vsmc}). We run the algorithm and use the
estimate to perform stochastic optimization of the surrogate \gls{ELBO}.

\paragraph{Stochastic Optimization.}
While the expectations in the surrogate \gls{ELBO} are still not
available in closed form, we can estimate it 
and its gradients with Monte Carlo. This admits
a stochastic optimization algorithm for finding the
optimal variational parameters of the \gls{VSMC} family. 

We assume the proposals ${r(x_t \given x_{t-1}; \lambda)}$ are
reparameterizable, \ie, we can simulate from~$r$ by
setting~${x_t = h(x_{t-1},\varepsilon_t\g\lambda), ~ \varepsilon_t \sim
  s(\varepsilon_t)}$ for some distribution $s$ not a function of
$\lambda$. With this assumption, rewrite the gradient of
\eqref{eq:lb} by using the reparameterization trick~\citep{Kingma2014,Rezende2014},
\begin{align}
\nabla \widetilde\Ls(\lambda) &= g_{\text{rep}}+ g_{\text{score}}\label{eq:rep-grad}\\
g_{\text{rep}} &= \E
\left[ \nabla\log \widehat{p}(y_{1:T})  \right], \nonumber\\
g_{\text{score}} &=\E
\left[ \log \widehat{p}(y_{1:T}) \nabla \log 
\widetilde\phi(a_{1:T-1}^{1:N}\given\varepsilon_{1:T}^{1:N}\g\lambda)\right].\nonumber
\end{align}
This expansion follows from the product rule, just as in the
generalized reparameterizations of~\citet{Ruiz2016}
and~\citet{Naesseth2017}.  
Note that all~$x_t^i$, implicit in the weights $w_t^i$ and 
$\widehat{p}(y_{1:T})$ are now replaced with their 
reparameterizations~$h(\cdot \g \lambda)$. The ancestor variables are discrete and cannot be
reparameterized---this can lead to high variance in the score function 
term, $g_{\text{score}}$ 
from \eqref{eq:rep-grad}. 

In Section~\ref{sec:expts}, we empirically assess
the impact of ignoring $g_{\text{score}}$ for optimization.
We empirically study optimizing with and without the score function term
for a small state space model where standard variance reduction 
techniques, explained below,
 are sufficient. We lower the variance
using Rao-Blackwellization~\citep{Robert2004,Ranganath2014}, noting that the
ancestor variables $a_{t-1}$ have no effect on weights prior to
time~$t$,
\begin{align}
&g_{\text{score}} = 
\nonumber\\
& \sum_{t=2}^{T}
\E
\left[  
\log \frac{\widehat{p}(y_{1:T})}{\widehat{p}(y_{1:t-1})}
\left(\sum_{i=1}^N \nabla \log \frac{w_{t-1}^{a_{t-1}^i}}{\sum_\ell w_{t-1}^\ell}\right)
 \right].\label{eq:rb}
\end{align}
Furthermore, we use the score 
function $\nabla \log 
\widetilde\phi(a_{1:T-1}^{1:N}\given\varepsilon_{1:T}^{1:N}\g\lambda)$ 
with an estimate of the future log average weights as
a control variate~\citep{Ranganath2014}. 

We found that ignoring the
score function term $g_{\text{score}}$~\eqref{eq:rb} from the ancestor variables, leads to faster 
convergence and very little difference in final \gls{ELBO}. This 
corresponds to approximating the gradient of $\widetilde\Ls$ by
\begin{align}
\nabla \widetilde\Ls(\lambda) &\approx \E
\left[ \nabla\log \widehat{p}(y_{1:T})  \right] = g_{\text{rep}}.
\label{eq:proxgrad}
\end{align}
This is the gradient we propose to use for optimizing the variational 
parameters of \gls{VSMC}.
See the supplementary material \ref{sec:supp:so} for more details, 
where we also provide a general score function-like 
estimator and the control variates.

\paragraph{Algorithm.} We now describe the full algorithm to optimize
the \gls{VSMC} variational approximation.  We form stochastic gradients
$\widehat \nabla \widetilde{\mathcal{L}}(\lambda)$ by estimating 
\eqref{eq:proxgrad} using a single sample from
$s(\cdot)\widetilde\phi(\cdot\given\cdot\g\lambda)$. The sample is
obtained as a byproduct of sampling \gls{VSMC}
(Algorithm~\ref{alg:vsmc}).  We use the step-size sequence Adam 
\citep{KingmaB2015} or $\rho^n$ proposed by \citet{Kucukelbir2016},
\begin{align}
    &\rho^n = \eta \cdot n^{-1/2 + \delta} \cdot \left(1 + 
    \sqrt{s^n}\right)^{-1}, \nonumber\\
    &s^n = t \left( \widehat \nabla \widetilde{\mathcal{L}}(\lambda^n) \right)^2 + (1-t) s^{n-1},
    \label{eq:stepsize}
\end{align}
where $n$ is the iteration number. We set $\delta = 10^{-16}$ and
$t=0.1$, and we try different values for $\eta$. 
Algorithm~\ref{alg:so} summarizes this optimization algorithm.\footnote{Reference implementation using Adam is available at 
\url{github.com/blei-lab/variational-smc}.}
\begin{algorithm}[t]
\caption{Stochastic Optimization for \gls{VSMC}}\label{alg:so}
\begin{algorithmic}[1]
\REQUIRE Data $y_{1:T}$, model $p(x_{1:T},y_{1:T})$, proposals 
$r(x_t\given x_{t-1}\g\lambda)$, number of particles $N$
\ENSURE Variational parameters $\lambda^\star$
\REPEAT
\STATE Estimate the gradient $\widehat \nabla \widetilde{\mathcal{L}}(\lambda^n)$ 
given by \eqref{eq:proxgrad}
\STATE Compute stepsize $\rho^n$ with \eqref{eq:stepsize}
\STATE Update $\lambda^{n+1} = \lambda^n + \rho^n \widehat \nabla \widetilde{\mathcal{L}}(\lambda^n)$
\UNTIL \textbf{convergence}
\end{algorithmic}
\end{algorithm}

\paragraph{Variational Expectation Maximization.}
Suppose the target distribution of
interest~${p(x_{1:T} \given y_{1:T}\g\theta)}$ has a set of unknown
parameters~$\theta$.  We can fit the parameters using \gls{VEM}
\citep{Beal2003}. The surrogate \gls{ELBO} is updated accordingly
\begin{align}
\log p(y_{1:T}\g\theta) \geq
\tilde{\mathcal{L}}(\lambda,\theta)\label{eq:lb:vem}
\end{align}
where the normalization constant~$p(y_{1:T}\g\theta)$ is now a
function of the parameters~$\theta$. Note that the expression for
$\tilde{\mathcal{L}}(\lambda,\theta)$ is exactly the same as
\eqref{eq:lb}, but where the weights (and potentially proposals) now
include a dependence on the model parameters $\theta$. Analogously,
the reparameterization gradients have the same form as
\eqref{eq:proxgrad}. We can maximize \eqref{eq:lb:vem},
with respect to both~$\theta$ and~$\lambda$, using stochastic
optimization. With data subsampling, \gls{VSMC} extends to large-scale datasets
of conditionally independent sequences \citep{Hoffman2013,titsias2014doubly}.



\section{Perspectives on Variational \gls{SMC}}
\label{sec:talking}
We give some perspectives on \gls{VSMC}. First, we consider the \gls{VSMC} special cases
of~${N=1}$ and~${T=1}$.  For $N=1$, \gls{VSMC} reduces to a structured
variational approximation:
there is no resampling and the variational distribution is exactly the
proposal. For $T=1$, \gls{VSMC} leads to a special case we call
\acrlong{VIS}, and a reinterpretation of the \gls{IWAE}
\citep{Burda2016}, which we explore further in the first half of this section.

Then, we think of sampling from \gls{VSMC} as sampling a highly
optimized \gls{SMC} approximation. This means many of the theoretical
\gls{SMC} results developed over the past 25 years 
can be adapted for \gls{VSMC}. We explore some examples in the second
half of this section. 

\paragraph{Variational Importance Sampling (VIS).}
The case where ${T=1}$ is \gls{SMC} without
any resampling, \ie, importance sampling. The
corresponding special case of \gls{VSMC} is \gls{VIS}. The
surrogate \gls{ELBO} for \gls{VIS} is exactly equal to the
\gls{IWAE} lower bound \citep{Burda2016}.

This equivalence provides new intuition behind the \gls{IWAE}'s variational approximation
on the latent variables. If we
want to make use of the approximation $q(x_{1:T}\g\lambda^\star)$
learned with the \gls{IWAE} lower bound, samples from the latent
variables should be generated with Algorithm~\ref{alg:vsmc}, \ie \gls{VIS}.
\begin{figure}[t]
    \centering
    \includegraphics[width=0.4\textwidth]{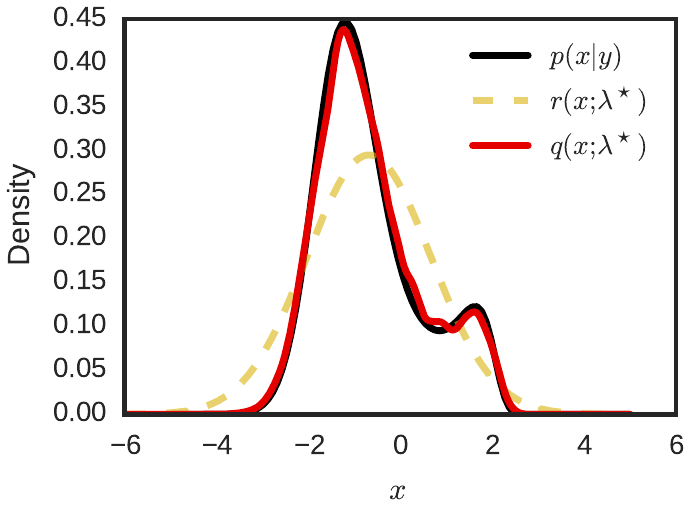}
\caption{Example of \gls{VIS} $q(x\g\lambda)$ approximating a multimodal 
$p(x\mid y)$ with a Gaussian proposal $r(x\g \lambda)$.}
\label{fig:vis}
\end{figure}
For \gls{VIS} it is possible to show that the surrogate \gls{ELBO} is 
always tighter than the one obtained by standard \gls{VB} (equivalent 
to \gls{VIS} with $N=1$) \citep{Burda2016}.
This result does not carry over to \gls{VSMC},
\ie we can find cases when the resampling creates a looser
bound compared to standard \gls{VB} or \gls{VIS}. However, in practice the 
\gls{VSMC} lower bound outperforms the
\gls{VIS} lower bound.

Figure~\ref{fig:vis} provides a simple example of \gls{VIS} 
applied to a multimodal ${p(x \given y)\propto 
\mathcal{N}(x\g 0,1) \, \mathcal{N}(y\g x^2/2, e^{x/2})}$ with a normal proposal 
${r(x\g\lambda)=\mathcal{N}(x\g\mu,\sigma^2)}$ and a kernel density estimate of the corresponding 
variational approximation $q(x\g\lambda)$. The number of 
particles is $N=10$. Standard \gls{VB} with a Gaussian approximation 
only captures one of the two modes; which one depends on the initialization.
We see that even a simple proposal can lead to a very flexible 
posterior approximation. This property is also inherited by the more 
general $T>1$ case, \gls{VSMC}.

\paragraph{Theoretical Properties.}
The normalization constant estimate of the \gls{SMC} sampler, 
$\widehat{p}(y_{1:T})$, is unbiased \citep{DelMoral2004,Pitt2012,Naesseth2014}. 
This, together with Jensen's inequality, implies that the 
surrogate \gls{ELBO}~$\E[\log \widehat p(y_{1:T})]$ is a lower 
bound to~$\log p(y_{1:T})$. If~$\log \widehat p(y_{1:T})$ is uniformly 
integrable it follows \citep{DelMoral2004}, as~$N\to\infty$, that
\begin{align*}
\widetilde{\mathcal{L}}(\lambda) =  \mathcal{L}(\lambda) = \log 
p(y_{1:T}).
\end{align*}
This fact means that the gap in Theorem~\ref{thm:lb} disappears and the distribution of 
the trajectory returned by \gls{VSMC} will 
tend to the true target distribution~${p(x_{1:T} \given y_{1:T})}$. 
A bound on the \gls{KL} divergence gives us the rate
\begin{align*}
\operatorname{KL}\left(q(x_{1:T}\g\lambda)\,\Big\|\,p(x_{1:T} \given y_{1:T})\right)
\leq \frac{c(\lambda)}{N},
\end{align*}
for some constant $c(\lambda) < \infty$. This is a special 
case of a ``propagation of chaos'' result from \citet[Theorem 
8.3.2]{DelMoral2004}.

We can arrive at this result informally by studying 
\eqref{eq:var-dist}: as the number of particles increases, the marginal likelihood
estimate will converge to the true marginal likelihood and the
variational posterior will converge to the true posterior. 
\citet{Huggins2017} provide further 
bounds on various divergences and metrics between \gls{SMC} and the 
target distribution.

\paragraph{\gls{VSMC} and $T$.} 
Like \gls{SMC}, \acrlong{VSMC} scales well with $T$. 
\citet{berard2014} show a central limit theorem for the \gls{SMC} 
approximation ${\log \widehat p(y_{1:T})-\log p(y_{1:T})}$ with~${N = b T}$, where~${b>0}$,
as~${T\to\infty}$. Under the same conditions as in that work, 
and assuming that~${\log \widehat p(y_{1:T})}$ is uniformly integrable, we 
can show that
\begin{align*}
&\operatorname{KL}\left(q(x_{1:T}\g\lambda)\,\Big\|\,p(x_{1:T} \given y_{1:T})\right)
\leq -\E\left[\log\frac{\widehat p(y_{1:T})}{p(y_{1:T})}\right]\\
&\qquad\xrightarrow[T\to\infty]{} \frac{\sigma^2(\lambda)}{2b}, \quad 0 < \sigma^2(\lambda) < \infty.
\end{align*}
The implication for \gls{VSMC} is significant.  We can make the
variational approximation \textit{arbitrarily accurate} by
setting~${N\propto T}$, even as~$T$ goes to infinity.  The supplement
shows that this holds in practice; see~\ref{sec:supp:independent} for
the toy example from Figure~\ref{fig:fig1}. We emphasize that neither
standard \gls{VB} nor \gls{IWAE} (\gls{VIS}) have this property.

\begin{figure*}[tb]
    \begin{subfigure}{0.49\textwidth}
        \setlength{\tabcolsep}{.1pt}
        \newcolumntype{M}{>{\centering\arraybackslash}m{\dimexpr.5\columnwidth}}
        \centering
        \begin{tabular}{M M}
            & \\
            \includegraphics[width=.5\columnwidth]{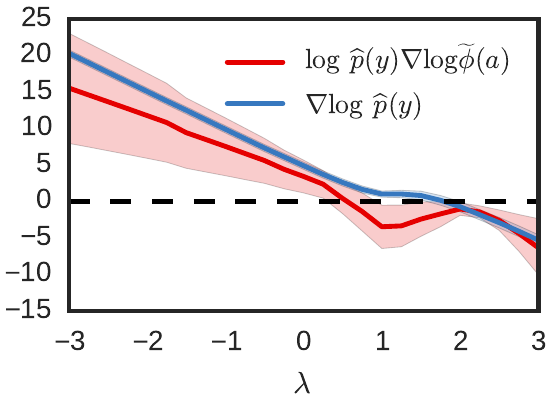} & 
            \includegraphics[width=.5\columnwidth]{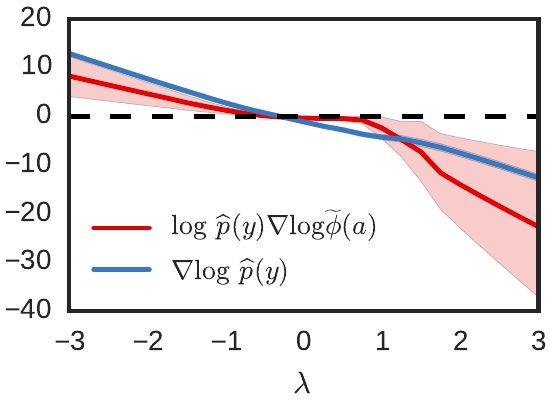} \\
            &\\
            \includegraphics[width=.5\columnwidth]{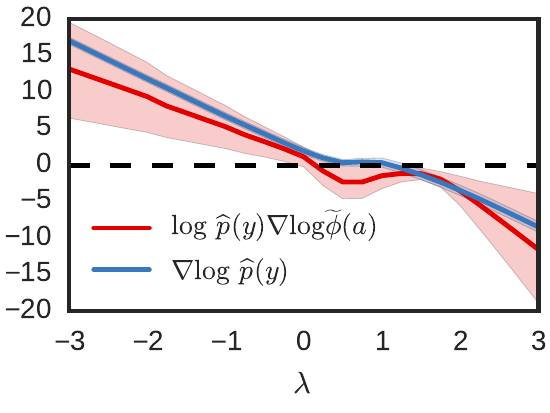} & 
            \includegraphics[width=.5\columnwidth]{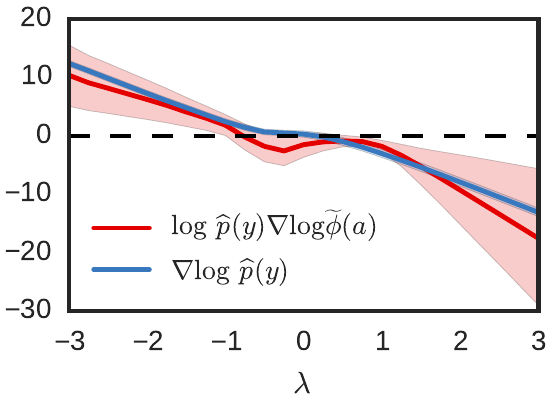} 
        \end{tabular}
    \end{subfigure}
    ~
    \begin{subfigure}{0.49\textwidth}
        \setlength{\tabcolsep}{.1pt}
        \newcolumntype{M}{>{\centering\arraybackslash}m{\dimexpr.5\columnwidth}}
        \centering
        \begin{tabular}{M M} 
            \footnotesize $d_x = 10, d_y=1, C$ sparse  & \footnotesize $d_x = 
            25, d_y=1, C$ dense\\
            \includegraphics[width=.5\columnwidth]{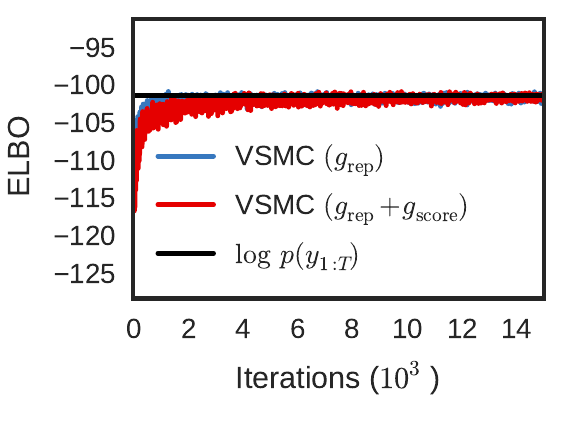} & 
            \includegraphics[width=.5\columnwidth]{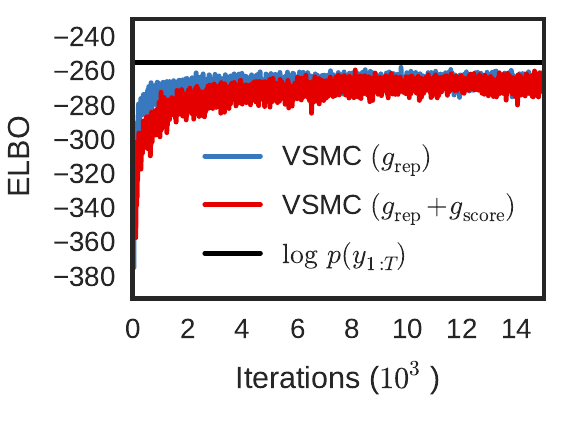} \\
            \footnotesize $d_x = 10, d_y=3, C$ dense  & \footnotesize $d_x = 
            25, d_y=25, C$ sparse\\
            \includegraphics[width=.5\columnwidth]{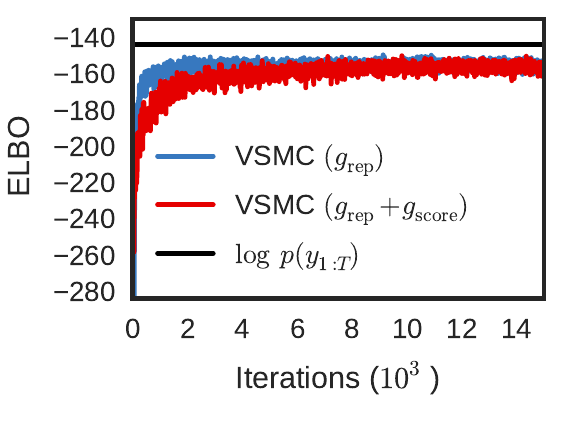} & 
            \includegraphics[width=.5\columnwidth]{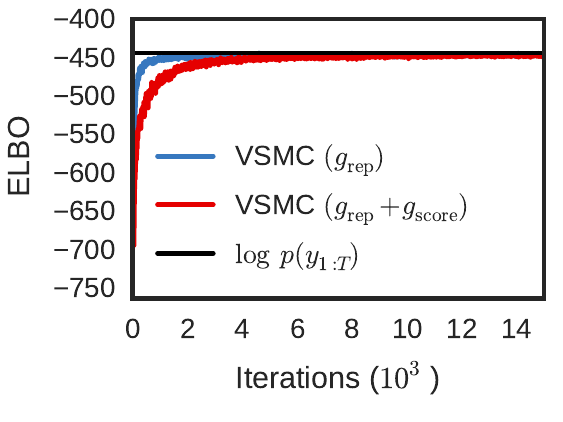} 
        \end{tabular}
    \end{subfigure}
    \caption{\emph{(Left)} Mean and spread of the stochastic gradient components
        $g_{\text{score}}$~\eqref{eq:rb} and $g_{\text{rep}}$~\eqref{eq:proxgrad}, 
        for the scalar linear Gaussian model 
        on four randomly generated datasets, 
        where the number of particles is $N=2$.
        \emph{(Right)} Log-marginal likelihood ($\log p(y_{1:T})$) and \gls{ELBO} 
        as a function of iterations for \gls{VSMC} 
        with biased gradients (blue) or unbiased gradients (red). 
        Results for four different linear Gaussian models.}\label{fig:bVSu}
\end{figure*}
\section{Empirical Study}\label{sec:expts}
\paragraph{Linear Gaussian State Space Model}
The linear Gaussian \gls{SSM} is a ubiquitous model of time
series data that enjoys efficient algorithms for computing the exact
posterior.  We use this model to study the convergence properties and 
impact of biased gradients for \gls{VSMC}. We further use it to 
confirm that we learn good proposals. We compare to the bootstrap 
particle filter (\gls{BPF}), which uses the prior as a proposal, and 
the (locally) optimal proposal that tilts the prior with the likelihood.

The model is
\begin{align*}
x_t &= A x_{t-1} +v_t, \\
y_t &= C x_t+e_t,
\end{align*}
where $v_t \sim \mathcal{N}(0,Q)$, $ e_t \sim \mathcal{N}(0,R)$, and
$x_1\sim \mathcal{N}(0,I)$. 
The log-marginal likelihood $\log p(y_{1:T})$
can be computed using the Kalman filter. 

We study the impact of the biased gradient \eqref{eq:proxgrad} for 
optimizing the surrogate \gls{ELBO} \eqref{eq:lb}.
First, consider a simple 
scalar model with~${A=0.5}$,~${Q=1}$, ${C=1}$, ${R=1}$, and ${T=2}$. For the 
proposal we use 
${r(x_t\given x_{t-1}\g\lambda)=\mathcal{N}(x_t\g\lambda+0.5 
x_{t-1},1)}$, with ${x_0 \equiv 0}$. Figure~\ref{fig:bVSu} (left) shows the mean 
and spread of estimates of $g_{\text{score}}$~\eqref{eq:rb}, with control variates, and 
$g_{\text{rep}}$~\eqref{eq:proxgrad}, as a function of $\lambda$ for four randomly 
generated datasets.
The optimal setting of $\lambda$ is where the sum of the 
means is equal to zero. Ignoring the score function term 
$g_{\text{score}}$~\eqref{eq:rb} will 
lead to a perturbation of the optimal $\lambda$. However, 
even for this simple model, the variance of the score 
function term (red) is several orders of magnitude 
higher than that of the reparameterization term (blue), despite
the variance reduction techniques of Section~\ref{sec:vsmc}. This 
variance has a significant 
impact on the convergence speed of the stochastic optimization. 

Next, we study the magnitude of the 
perturbation, and its effect on the surrogate \gls{ELBO}.
We generate data with $T=10$, $(A)_{ij}= \alpha^{|i-j|+1}$ for 
$\alpha=0.42$, $Q=I$, and $R=I$. We explored several 
settings of $d_x=\dim(x_t)$, $d_y=\dim(y_t)$, and $C$. Sparse $C$ 
measures the first $d_y$ components of $x_t$, and dense $C$ has 
randomly generated elements $C_{ij} \sim \mathcal{N}(0,1)$. 
Figure~\ref{fig:bVSu} (right) shows the true log-marginal likelihood and 
\gls{ELBO} as a function of iteration. It shows \gls{VSMC} with biased gradients
(blue) and unbiased gradients (red). We choose the proposal
\begin{align*}
r(x_t\mid x_{t-1}\g\lambda) &= \mathcal{N}\left(x_t\mid \mu_t + 
\operatorname{diag}(\beta_t) A x_{t-1}, 
\operatorname{diag}(\sigma_t^2)\right).
\end{align*}
with $\lambda = \left\{ \mu_t,\beta_t,\sigma_t^2\right\}_{t=1}^T$, and 
set the number of particles to $N=4$. Note that while the 
gradients are biased, the resulting \gls{ELBO} is not.
We can see that the final \gls{VSMC} \gls{ELBO} values are very 
similar, regardless of whether we train with biased or unbiased gradients. 
However, biased gradients converge faster. Thus, we use biased gradients
in the remainder of our experiments.

Next, we study the effect of learning the proposal using \gls{VSMC} 
compared with standard proposals in the \gls{SMC} literature. The most 
commonly used is the \gls{BPF}, sampling from the 
prior $f$. We also consider the so-called optimal proposal, $r \propto 
f \cdot g$, which minimizes the variance of the 
incremental importance weights \citep{doucet2009tutorial}. Table~\ref{tab:vsmcVSopt}
shows results for a linear Gaussian \gls{SSM} when 
${T=25}$, ${Q = 0.1^2 I}$, ${R=1}$, ${d_x=10}$, and~${d_y=1}$. Because of the 
relatively high-dimensional state, \gls{BPF} exhibits significant bias
whereas the optimal proposal \gls{SMC} performs much better. 
\gls{VSMC} outperforms them both, learning an accurate proposal 
that results in an \gls{ELBO} only $0.9$ nats lower than the true 
log-marginal likelihood. We further emphasize that the optimal proposal 
is unavailable for most models.
\begin{table}[tb]
\centering
\caption{\gls{ELBO} for \gls{BPF}, \gls{SMC} with (locally) optimal 
proposal, and \gls{VSMC}. The true log-marginal likelihood is given 
by $\log p(y_{1:T}) = -236.9$.}
\label{tab:vsmcVSopt}
\begin{tabular}{cccc}
& \gls{BPF} & Optimal \gls{SMC} & \gls{VSMC} \\
\gls{ELBO} & $-6701.4$ & $-253.4$ & $\mathbf{-237.8}$
\end{tabular}
\end{table}

\paragraph{Stochastic Volatility}
A common model in financial econometrics is the (multivariate) 
stochastic volatility model~\citep{Chib2009}.
The model is
\begin{align*}
x_t &= \mu +\phi(x_{t-1}-\mu)+v_t, \\
y_t &= \beta\exp{\left(\frac{x_t}{2}\right)}e_t,
\end{align*}
where ${v_t \sim \mathcal{N}(0,Q)}$, ${e_t \sim \mathcal{N}(0,I)}$, 
${x_1\sim \mathcal{N}(\mu,Q)}$, and ${\theta = (\mu, \phi, Q, \beta)}$. 
(In the multivariate case, multiplication is element-wise.) Computing 
${\log p(y_{1:T}\g\theta)}$ and its gradients for this model is 
intractable, we study the \gls{VEM} approximation to find the 
unknown parameters $\theta$. We compare \gls{VSMC} with \gls{IWAE} 
and
structured \gls{VI}. 
For the proposal in \gls{VSMC} and \gls{IWAE} we choose
\begin{align*}
r(x_t \given x_{t-1}\g\lambda,\theta) \propto f(x_t\given 
x_{t-1}\g\theta) \, \mathcal{N}(x_t\g \mu_t, \Sigma_t),
\end{align*}
with variational parameters 
${\lambda = (\mu_1,\ldots,\mu_T, \Sigma_1, \ldots, \Sigma_T)}$. We 
define the variational approximation for structured \gls{VI} to be
$q(x_{1:T}\g\lambda,\theta) = \prod_{t=1}^T r(x_t \given x_{t-1}\g\lambda,\theta)$. 

We study $10$ years of monthly returns ($9/2007$ to $8/2017$) for the 
exchange rate of $22$ international currencies with respect to US 
dollars. The data is from the Federal Reserve 
System. 
Table~\ref{tab:sv} reports the optimized \gls{ELBO} (higher is 
better) for different settings of the number of particles/samples 
$N=\{4,8,16\}$. \gls{VSMC} outperforms the competing methods 
with almost $0.2$ nats per time-step.
\begin{table}[tb]
\centering
\caption{\gls{ELBO} for the stochastic volatility model with ${T=119}$ 
on exchange rate data. We compare \gls{VSMC} (this paper) with 
\gls{IWAE} and structured \gls{VI}.}
\label{tab:sv}
\begin{tabular}{ccc}
&Method & \gls{ELBO} \\
\rule{0pt}{3ex}
& Structured \gls{VI} & $6905.1$ \\
\hline
\multirow{2}{*}{$N=4$} & \gls{IWAE} & $6911.2$ \\
 & \textbf{\gls{VSMC}} & $\mathbf{6921.6}$ \\
\hline
 \multirow{2}{*}{$N=8$} & \gls{IWAE} & $6912.4$ \\
 & \textbf{\gls{VSMC}} & $\mathbf{6935.8}$ \\
\hline
 \multirow{2}{*}{$N=16$} & \gls{IWAE} & $6913.3$ \\
 & \textbf{\gls{VSMC}} & $\mathbf{6936.6}$ \\
\end{tabular}
\end{table}

In theory we can improve the bound of both \gls{IWAE} and 
\gls{VSMC} by increasing the number of samples $N$. 
This means we can first learn proposals using 
only a few particles $N$, for computational efficiency. Then, at test 
time, we can increase $N$ as needed for improved accuracy.
We study the impact of increasing the number of samples for \gls{VSMC} 
and \gls{IWAE} using fix $\theta^\star$ and $\lambda^\star$ optimized 
with $N=16$. Figure~\ref{fig:sv} shows that the gain for \gls{IWAE} is 
limited, whereas for \gls{VSMC} it can be significant.

\begin{figure}[h]
  \centering
  \includegraphics[width=2.5in]{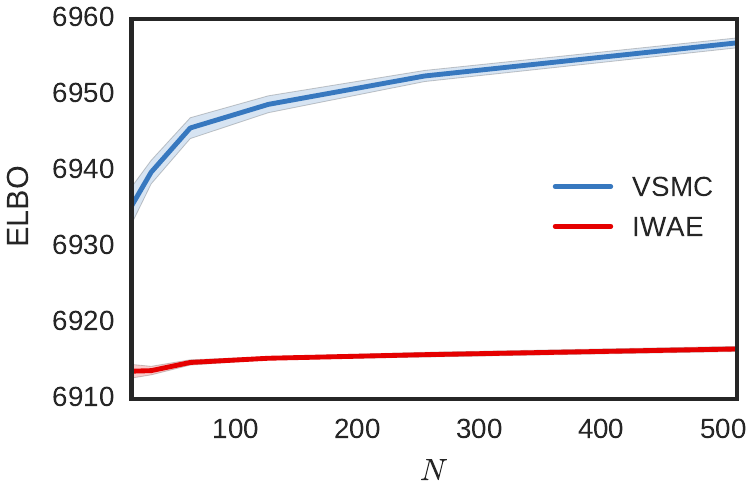}
  \vspace{-1em}
  \caption{The estimated \gls{ELBO} for \gls{VSMC} (this paper) and 
\gls{IWAE}
, with confidence bands, as a 
function of the number of particles $N$ for fix $\theta^\star, 
~\lambda^\star$.}\label{fig:sv}
\end{figure}

%
\paragraph{Deep Markov Model}

An important problem in neuroscience is understanding
dynamics of neural circuits. We study a population of 105
motor cortex neurons simultaneously recorded in a macaque monkey as it
performed reaching movements~\citep[c.f.][]{Gao2016}.  In each trial, the
monkey reached toward one of fourteen targets; each trial is~${T=21}$
time steps long. We train on~${700}$ trials and test on~$84$.

We use recurrent neural networks to model both the 
dynamics and observations. The model is
\begin{align*}
x_t &= \mu_\theta(x_{t-1})+\exp\left(\sigma_\theta (x_{t-1})/2\right) v_t, \\
y_t &\sim \operatorname{Poisson}\left(\exp{\left(\eta_\theta(x_t)\right)}\right),
\end{align*}
where $v_t \sim \mathcal{N}(0,I)$, $x_0 \equiv 0$, and $\mu, \sigma, 
\eta$ are neural networks parameterized by $\theta$. The 
multiplication in the transition dynamics is element-wise. This is a 
deep Markov model \citep{Krishnan2017}.

For inference we use the following proposal for both 
\gls{VSMC} and \gls{IWAE},
\begin{align*}
r(x_t\given x_{t-1}, y_t\g\lambda) &\propto 
\mathcal{N}\left(x_t\g 
\mu_\lambda^x(x_{t-1}),\exp\left(\sigma_\lambda^x(x_{t-1})\right)\right) \\
&\times \,
\mathcal{N}\left(x_t\g 
\mu_\lambda^y(y_t),\exp\left(\sigma_\lambda^y(y_{t})\right)\right),
\end{align*}
where~$\mu^x,\sigma^x,\mu^y,\sigma^y$ are neural networks 
parameterized by~$\lambda$, and the proposal factorizes over the 
components of~$x_t$.
\begin{figure}[tb]
\centering
\includegraphics[width=.8\columnwidth]{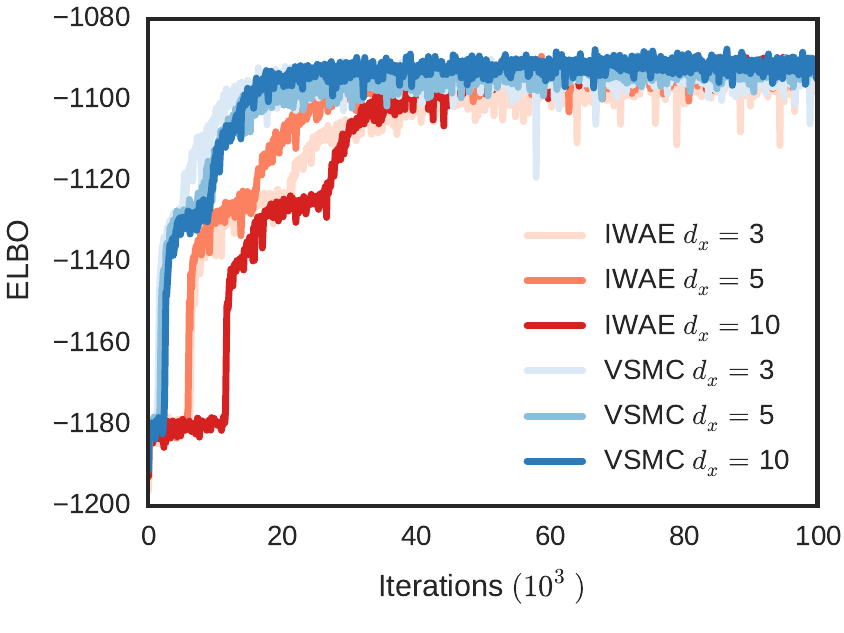}
\caption{The estimated \gls{ELBO} of the neural population test data as 
a function of iterations for \gls{VSMC} (this paper) and \gls{IWAE}
, for ${d_x =\{3,5,10\}}$ and $T=21$.}\label{fig:neural}
\end{figure}
Figure~\ref{fig:neural} illustrates the result for~${d_x = \{3,5,10\}}$ 
with~${N=8}$. \gls{VSMC}
gets to the same \gls{ELBO} faster.


\glsresetall
\section{Conclusions}

We introduced the \gls{VSMC} family, a new variational approximating 
family that
provides practitioners with a flexible, accurate, and powerful
approximate Bayesian inference algorithm. \gls{VSMC} melds \gls{VI}
and \gls{SMC}. This results in a variational approximation that lets us 
trade-off fidelity to the posterior with computational complexity.

\section*{Acknowledgements}
Christian A.\ Naesseth is supported by CADICS, a Linnaeus Center, 
funded by the Swedish Research Council (VR). 
Scott W. Linderman is supported by the Simons Foundation SCGB-418011.
This work is supported by ONR N00014-11-1-0651, DARPA
PPAML FA8750-14-2-0009, the Alfred P. Sloan Foundation, 
and the John Simon Guggenheim Foundation.

\bibliography{references}
\bibliographystyle{abbrvnat}

\newpage
\appendix
\onecolumn
\section{Variational Sequential Monte Carlo -- Supplementary Material}
\subsection{Proof of Proposition~\ref{prop:dist}}\label{sec:pf:dist}
We start by noting that the distribution of all random variables 
generated by the \gls{VSMC} algorithm is given by
\begin{align}
&\widetilde\phi(x_{1:T}^{1:N},a_{1:T-1}^{1:N},b_T \g\lambda) = \frac{w_T^{b_T}}{\sum_{\ell}w_T^\ell} 
\prod_{i=1}^N r(x_1^i\g\lambda) \cdot \prod_{t=2}^T \prod_{i=1}^N
\frac{w_{t-1}^{a_{t-1}^i}}{\sum_\ell w_{t-1}^\ell} r(x_t^i | 
x_{t-1}^{a_{t-1}^i}\g\lambda).\label{eq:supp:vsmc-all}
\end{align}
We are interested in the marginal distribution 
$q(x_{1:T}\g\lambda)\triangleq \widetilde\phi(x_{1:T}\g\lambda) = 
\E_{b_{1:T}}[\widetilde\phi(x_{1:T}^{b_{1:T}},b_{1:T}\g\lambda)]$.
A key observation is that the distribution of $b_{1:T}\mid x_{1:T}$, 
the conditional distribution of the ancestral path of the returned particle, 
is uniform on $\left\{1,\ldots,N\right\}^T$. 
Thus we get
\begin{align}
&q(x_{1:T}\g\lambda) = 
\frac{\widetilde\phi(x_{1:T}^{b_{1:T}},b_{1:T}\g\lambda)}{\widetilde\phi(b_{1:T}\mid x_{1:T}\g\lambda)} 
=\frac{1}{N^{-T}}  \sum_{a_{1:T-1}^{-b_{1:T-1}}}\int 
\widetilde\phi(x_{1:T}^{b_{1:T}},x_{1:T}^{\neg 
b_{1:T}},a_{1:T-1}^{\neg b_{1:T-1}} \g\lambda)
\,\myd x_{1:T}^{\neg b_{1:T}},
\label{eq:supp:vsmc}
\end{align}
where
\begin{align*}
&\frac{1}{N^{-T}}  \widetilde\phi(x_{1:T}^{b_{1:T}},x_{1:T}^{\neg b_{1:T}},a_{1:T-1}^{\neg b_{1:T-1}} \g\lambda) \\
&= N^T  
\frac{w_1^{b_1}}{\sum_\ell w_1^\ell}r(x_1^{b_1}\g\lambda)\prod_{t=2}^T 
\frac{w_t^{b_t}}{\sum_\ell w_t^\ell} r(x_t^{b_t}\mid 
x_{t-1}^{b_{t-1}}\g\lambda) \cdot
\prod_{\substack{i=1 \\ i\neq b_1}}^N r(x_1^i\g\lambda) \cdot \prod_{t=2}^T 
\prod_{\substack{i=1 \\ i\neq b_t}}^N
\frac{w_{t-1}^{a_{t-1}^i}}{\sum_\ell w_{t-1}^\ell} r(x_t^i | 
x_{t-1}^{a_{t-1}^i}\g\lambda) \\
&= p(x_1^{b_1},y_1)\prod_{t=2}^T 
\frac{p(x_{1:t}^{b_{1:t}},y_{1:t})}{p(x_{1:t-1}^{b_{1:t-1}},y_{1:t-1})}
\prod_{t=1}^T \frac{1}{\frac{1}{N} \sum_\ell w_t^\ell} \cdot
\prod_{\substack{i=1 \\ i\neq b_1}}^N r(x_1^i\g\lambda) \cdot \prod_{t=2}^T 
\prod_{\substack{i=1 \\ i\neq b_t}}^N
\frac{w_{t-1}^{a_{t-1}^i}}{\sum_\ell w_{t-1}^\ell} r(x_t^i | 
x_{t-1}^{a_{t-1}^i}\g\lambda) \\
&= p(x_{1:T}^{b_{1:T}},y_{1:T}) \prod_{t=1}^T \frac{1}{\frac{1}{N} \sum_\ell w_t^\ell}
\cdot \widetilde\phi(x_{1:T}^{\neg b_{1:T}},a_{1:T-1}^{\neg b_{1:T-1}} \g\lambda).
\end{align*}
We insert the above expression in \eqref{eq:supp:vsmc} and we get
\begin{align}
q(x_{1:T}\g\lambda) &= p(x_{1:T}^{b_{1:T}},y_{1:T})
\sum_{a_{1:T-1}^{\neg b_{1:T-1}}}\int \left(
\prod_{t=1}^T \frac{1}{N} \sum_{i=1}^N w_t^i \right)^{-1}
\cdot \widetilde\phi(x_{1:T}^{\neg b_{1:T}},a_{1:T-1}^{\neg b_{1:T-1}} \g\lambda)
\,\myd x_{1:T}^{\neg b_{1:T}} \nonumber \\
&= p(x_{1:T}^{b_{1:T}},y_{1:T})
\E_{\widetilde\phi(x_{1:T}^{\neg b_{1:T}},a_{1:T-1}^{\neg b_{1:T-1}} \g\lambda)}
\left[ \left(
\prod_{t=1}^T \frac{1}{N} \sum_{i=1}^N w_t^i \right)^{-1}\right].
\end{align}
\hfill$\square$

\subsection{Proof of Theorem~\ref{thm:lb}}\label{sec:pf:lb}
The \gls{ELBO}, using the above result about the distribution of 
$q(x_{1:T}\g\lambda)$, is given by
\begin{align}
\mathcal{L}(\lambda) &= \E_{q(x_{1:T}\g\lambda)}\left[\log 
p(x_{1:T},y_{1:T}) -\log q(x_{1:T}\g\lambda)\right] 
\nonumber \\
&= -\int \left\{p(x_{1:T}^{b_{1:T}},y_{1:T}) 
\E_{\widetilde\phi(x_{1:T}^{\neg b_{1:T}},a_{1:T-1}^{\neg b_{1:T-1}} \g\lambda)}
\left[ \frac{1}{
\prod_{t=1}^T \frac{1}{N} \sum_{i=1}^N w_t^i }\right] \cdot \nonumber \right.\\
&\left.\qquad\qquad\qquad\cdot 
\log  
\E_{\widetilde\phi(x_{1:T}^{\neg b_{1:T}},a_{1:T-1}^{\neg b_{1:T-1}} \g\lambda)}
\left[ \frac{1}{
\prod_{t=1}^T \frac{1}{N} \sum_{i=1}^N w_t^i }\right] \right\}
\myd x_{1:T}^{b_{1:T}}.
\label{eq:supp:exact}
\end{align}
Note that $-t \log t$ is a concave function for $t>0$, this means by 
the conditional Jensen's inequality we have $-\E[t]\log \E[t] \geq 
-\E[t\log t]$. If we apply this to \eqref{eq:supp:exact} we get
\begin{align*}
\mathcal{L}(\lambda) &\geq \int 
\E_{\widetilde\phi(x_{1:T}^{\neg b_{1:T}},a_{1:T-1}^{\neg b_{1:T-1}} \g\lambda)}
\left[ \frac{p(x_{1:T}^{b_{1:T}},y_{1:T})}{
\prod_{t=1}^T \frac{1}{N} \sum_{i=1}^N w_t^i }
  \sum_{t=1}^T \log 
\left(\frac{1}{N}\sum_{i=1}^N w_t^i\right)\right]\myd 
x_{1:T}^{b_{1:T}} \nonumber\\
&= \E_{\widetilde\phi(x_{1:T}^{1:N},a_{1:T-1}^{1:N}\g\lambda)}\left[\sum_{t=1}^T \log 
\left(\frac{1}{N}\sum_{i=1}^N w_t^i\right)\right] = 
\widetilde{\mathcal{L}}(\lambda),
\end{align*}
where the last step follows because $q(x_{1:T}\g\lambda)$ is 
the marginal of $\widetilde\phi(x_{1:T}^{1:N},a_{1:T-1}^{1:N}\g\lambda)$.
\hfill$\square$

\subsection{Stochastic Optimization}\label{sec:supp:so}

For the control variates we use
\begin{align*}
\sum_{t=2}^T c_t \E_{s(\cdot)\widetilde\phi(\cdot\mid\cdot\g\lambda)}
\left[\sum_{i=1}^N \nabla \log w_{t-1}^{a_{t-1}^i} - 
\sum_{\ell=1}^N \frac{w_{t-1}^\ell}{\sum_m w_{t-1}^m} \nabla \log 
w_{t-1}^\ell \right]
\end{align*}
where
\begin{align*}
c_t = \E_{s(\cdot)\widetilde\phi(\cdot\mid\cdot\g\lambda)}
\left[  \sum_{t'=t}^T 
\log\left(\frac{1}{N}\sum_{i=1}^N 
w_{t'}^i\right) \right].
\end{align*}
In practice we use a stochastic estimate of $c_t$.

For $T=2$ we can use a leave-one-out estimator of the ancestor variable 
score function gradient
\begin{align*}
\sum_{i=1}^N \E_{s(\cdot)\widetilde\phi(\cdot\mid\cdot\g\lambda)}
\left[\log \left(\frac{N-1}{N}\frac{\sum_{\ell=1}^N w_2^\ell}{\sum_{j\neq i}w_2^j}\right) \left(\nabla \log w_{1}^{a_{1}^i} - 
\sum_{\ell=1}^N \frac{w_{1}^\ell}{\sum_m w_{1}^m} \nabla \log 
w_{1}^\ell \right)\right].
\end{align*}

\paragraph{Score Function Gradient}\label{sec:scfnc}
Below we provide the derivation of a score function-like estimator 
that is applicable in very general settings. However, we have found 
that in practice the variance tends to be quite high.
\begin{align*}
\nabla \widetilde\Ls(\lambda) &= \nabla 
\E_{\widetilde\phi(x_{1:T}^{1:N},a_{1:T-1}^{1:N}\g\lambda)}\left[\log 
\widehat{p}(y_{1:T})\right] \\
&= \E_{\widetilde\phi(x_{1:T}^{1:N},a_{1:T-1}^{1:N}\g\lambda)}\left[
\nabla \log \widehat{p}(y_{1:T}) + 
\log \widehat{p}(y_{1:T}) \nabla \log \widetilde\phi(x_{1:T}^{1:N},a_{1:T-1}^{1:N}\g\lambda)
\right],
\end{align*}
with
\begin{align*}
&\nabla \log \widehat{p}(y_{1:T}) = \nabla \sum_{t=1}^T  
\log\left(\frac{1}{N}\sum_{i=1}^N w_t^i\right) = \sum_{t=1}^T\sum_{i=1}^N 
\frac{w_t^i}{\sum_\ell w_t^\ell}\nabla \log w_t^i,
\end{align*}
and
\begin{align*}
&\nabla \log \widetilde\phi(x_{1:T}^{1:N},a_{1:T-1}^{1:N}\g\lambda) \\ 
&= 
\sum_{i=1}^N \left[\nabla \log r(x_1^i\g\lambda) + \sum_{t=2}^T \left[\nabla \log
r(x_t^i|x_{t-1}^{a_{t-1}^i}\g\lambda) + \nabla \log w_{t-1}^{a_{t-1}^i} - 
\sum_{\ell=1}^N \bar w_{t-1}^\ell \nabla \log w_{t-1}^\ell \right] \right].
\end{align*}

\begin{figure}[tb]
    \centering
    \begin{subfigure}[b]{0.48\textwidth}
        \includegraphics[width=1.1\textwidth]{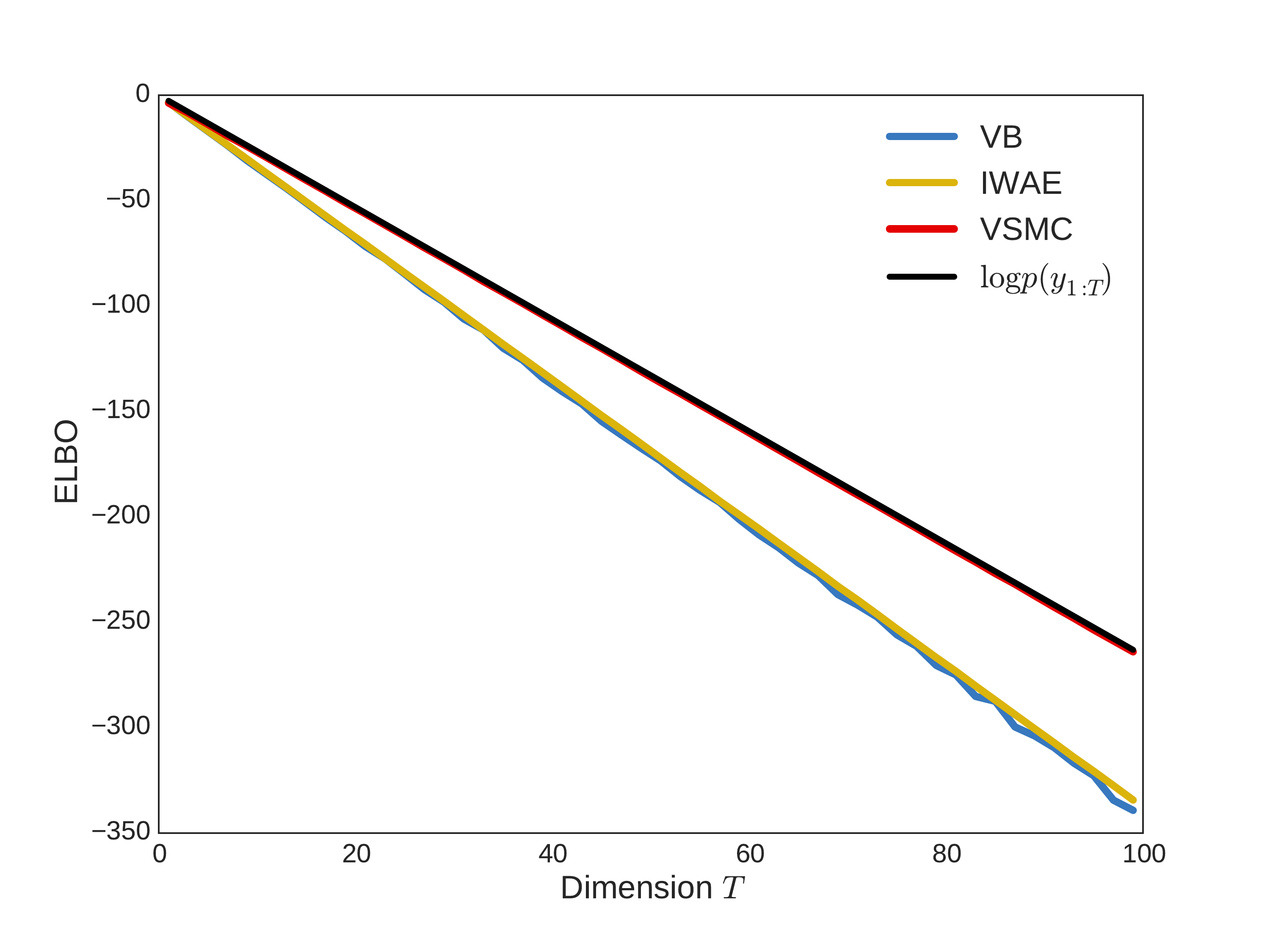}
        \caption{$\lambda_{\acrshort{IWAE}} = \lambda_{\acrshort{VB}}^\star$}
        \label{fig:sameVB}
    \end{subfigure}
    ~ 
    \begin{subfigure}[b]{0.48\textwidth}
        \includegraphics[width=1.1\textwidth]{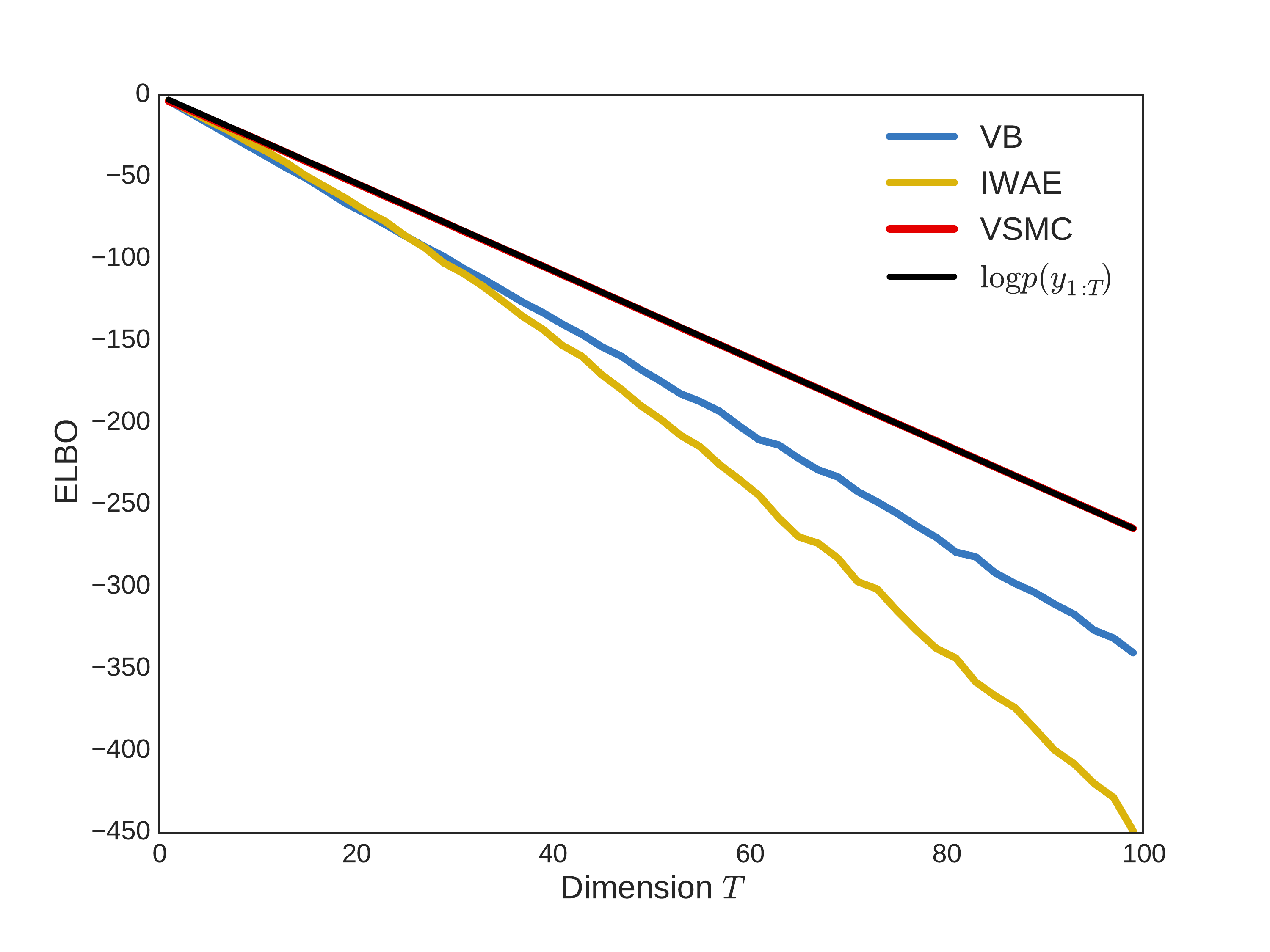}
        \caption{$\lambda_{\acrshort{IWAE}} = \lambda_{\acrshort{VSMC}}^\star$}
        \label{fig:sameVSMC}
    \end{subfigure}
    \caption{\gls{ELBO}, for standard \acrshort{VB}, \acrshort{IWAE}, and \acrshort{VSMC},
as a function of the dimension $T$ of a toy problem. Here we set the number 
of samples in \acrshort{IWAE} and \acrshort{VSMC} to be $N=2T$.}\label{fig:supp:T}
\end{figure}

\subsection{Scaling With Dimension}\label{sec:supp:independent}
In this section we study how the methods compare on a simple toy model 
defined by
\begin{align*}
p(x_{1:T},y_{1:T}) &= \prod_{t=1}^T \mathcal{N}(x_t\g 
0,1)\mathcal{N}(y_t\g x_t^2,1).
\end{align*}
We study the data set $y_t = 3, \forall t$. Figure~\ref{fig:supp:T} 
shows the result when we let the number of samples in \gls{IWAE} 
(\gls{VIS}) and \gls{VSMC} grow with the dimension $N=2T$. For low $T$ the optimal 
parameters for \gls{IWAE} are close to $\lambda_{\gls{VSMC}}^\star$. 
On the other hand for high $T$, the optimal 
parameters for \gls{IWAE} are close to those of standard \gls{VB}, \ie
$\lambda_{\gls{VB}}^\star$. Figure~\ref{fig:supp:T} indicates that just by letting 
$N\propto T$, \gls{VSMC} can achieve arbitrarily good approximation of 
$p(x_{1:T}\mid y_{1:T})$ even if $T\to\infty$. This holds, under some 
regularity conditions, even if $p(x_{1:T},y_{1:T})$ is a state space 
model \citep{berard2014}. This asymptotic approximation property is not 
satisfied by \gls{VIS}, we see in Figure~\ref{fig:supp:T} that the 
approximation deteriorates as $T$ increases.
Note that this does not hold if the dimension of the latent 
space, \ie $\operatorname{dim}(x_t)$, tends to infinity rather than the number of time 
points $T$.

\end{document}